\documentclass[12pt]{article}
\usepackage{framed,epsf,latexsym,amsmath,amssymb,amscd,amsthm,textcomp}
\usepackage[numbers]{natbib}
\usepackage{wrapfig}
\usepackage{graphicx}
\usepackage[pdftex,colorlinks]{hyperref}
\usepackage[margin=1.0in]{geometry}
\usepackage{tikz}

\usepackage{algorithm}
\usepackage{algorithmic}

\newtheorem{theorem}{Theorem}
\newtheorem{lemma}[theorem]{Lemma}

\newtheorem{remark}{Comment}

\newtheorem{example}{Example}

\newtheorem{counter-example}[theorem]{Counter example}

\newtheorem{open question}[theorem]{Open question}
\newtheorem{corollary}[theorem]{Corollary}

\newcommand{\In}{{\rm in}}
\newcommand{\out}{{\rm out}}
\newcommand{\rfss}{{\tt RFSS}}
\newcommand{\depth}{{\rm depth}}
\newcommand{\eqdef}{\stackrel{\footnotesize\rm def}{=}}

\newcommand{\cd}{{\cal D}}

\newcommand{\cc}{{\cal C}}

\newcommand{\ch}{{\cal H}}

\newcommand{\cx}{{\cal X}}
\newcommand{\cy}{{\cal Y}}

\newcommand{\cs}{{\cal S}}

\newcommand{\x}{{\mathbf x}}
\newcommand{\ba}{{\mathbf a}}

\newcommand{\z}{{\mathbf z}}
\newcommand{\w}{{\mathbf w}}
\newcommand{\f}{{\mathbf f}}
\newcommand{\g}{{\mathbf g}}
\newcommand{\e}{{\mathbf e}}
\newcommand{\bv}{{\mathbf v}}

\newcommand{\vomega}{{\boldsymbol{\omega}}}
\newcommand{\valpha}{{\boldsymbol{\alpha}}}
\newcommand{\vbeta}{{\boldsymbol{\beta}}}

\renewcommand{\Re}{\mathrm{Re}}

\newcommand{\reals}{{\mathbb R}}
\newcommand{\torus}{{\mathbb T}}
\newcommand{\sphere}{{\mathbb S}}

\newcommand{\complex}{{\mathbb C}}

\newcommand{\inner}[1]{\langle #1 \rangle}

\newcommand{\todonow}[1]{{\bf #1}}

\DeclareMathOperator*{\E}{\mathbb{E}}

\def\nrme{{norm-efficient}}

\title{Random Features for Compositional Kernels}
\author{
Amit Daniely\thanks{Google Brain, \tt{\{amitdaniely, frostig, vineet, singer\}@google.com}} \and
Roy Frostig \and
Vineet Gupta \and
Yoram Singer}
\begin{document}

\setcounter{page}{0}
\maketitle

\thispagestyle{empty}
\begin{abstract}
 We describe and analyze a simple random feature scheme (RFS) from prescribed
 compositional kernels. The compositional kernels we use are inspired by the
 structure of convolutional neural networks and kernels. The resulting scheme
 yields sparse and efficiently computable features. Each random
 feature can be represented as an algebraic expression over a {\em small}
 number of (random) paths in a composition tree. Thus, compositional random
 features can be stored compactly. The discrete nature of the generation
 process enables de-duplication of repeated features, further compacting the
 representation and increasing the diversity of the embeddings. Our approach
 complements and can be combined with previous random feature schemes.
\end{abstract}

\newpage

\section{Introduction} \label{intro:sec}
Before the resurgence of deep architectures, kernel methods~\cite{Vapnik95,
Vapnik98, ScholkopfBuSm98} achieved state of the art results in various
supervised learning tasks. Learning using kernel representations amounts to
convex optimization with provable convergence guarantees. The first generation
of kernel functions in machine learning were oblivious to spatial or temporal
characteristics of input spaces such as text, speech, and images. A natural
way to capture local spatial or temporal structure is through hierarchical
structures using compositions of kernels, see for
instance~\cite{scholkopf1998prior, grauman2005pyramid}. Compositional kernels
became more prominent among kernel methods following the success of deep
networks and, for several tasks, they currently achieve the state of the art
among all kernel methods~\cite{cho2009kernel, bo2011object,
mairal2014convolutional, Mairal16}.

While the ``kernel trick'' unleashes the power of convex optimization, it
comes with a large computational cost as it requires storing or repeatedly
computing kernel products between pairs of examples.  \citet{RahimiRe07}
described and analyzed an elegant and computationally effective way that
mitigates this problem by generating random features that approximate certain
kernels. Their work was extended to various other kernels~\cite{kar2012random,
pennington2015spherical, bach2015equivalence, bach2014breaking}.

In this paper we describe and analyze a simple random feature generation
scheme from prescribed compositional kernels. The compositional kernels we use
are inspired by the structure of neural networks. The kernels' definition and the connection to networks
was developed in~\cite{DanielyFrSi16, daniely2017}. Our feature map construction has several
benefits over previous ones. It naturally exploits hierarchical
structure in terms of representation power. The random feature generation is
computationally efficient. More importantly, computing the feature map
 is efficient and often can be performed in time
linear in the embedding dimension. Last but not least, computing the feature
map requires highly sparse access patterns to the input, implying low memory
requirements in the process.

The course of the paper is as follows. After a brief background, we start the
paper by recapping in Sec.~\ref{rfs:sec} the notion of random features schemes
(RFS). Informally speaking, a random feature scheme is an embedding from an
input space into the real or complex numbers. The scheme is random such that
multiple instantiations result in different mappings. Standard inner products
in the embedded space emulate a kernel function and converge to the inner
product that the kernel defines.  We conclude the section with a
derivation of concentration bounds for kernel approximation by RFS and a
generalization bound for learning with RFS.

The subsequent sections provide the algorithmic core of the paper. In
Sec.~\ref{rfs4basic:sec} we describe RFS for basic spaces such as
$\{-1,+1\}$, $[n]$, and $\torus$. We show that the standard inner product on
the sphere in one and two dimensions admits an effective \nrme\ RFS.  However,
any RFS for $\sphere^{d-1}$ where $d \geq 3$ is norm-deficient.  In
Sec.~\ref{comp_rfs:sec}, we discuss how to build random
feature schemes from compositional kernels that are described by a {\em computation skeleton},
which is an annotated directed acyclic graph.  The base spaces constitute the
initial nodes of the skeleton.  As the name implies, a compositional kernel
consists of a succession of compositions of prior constructed kernels, each of
which is by itself a compositional kernel or a base kernel. We conclude the
section with run-time and sparsity-level analysis.

The end result of our construction is a lightweight yet flexible feature generation
procedure. Each random feature can be represented as an algebraic expression
over of a {\em small} number of (random) paths in a composition tree. Thus,
compositional random features can be stored very compactly. The discrete
nature of the generation process enables de-duplication of repeated features,
further compacting the representation and increasing the diversity of the
embeddings. The latter property cannot be directly achieved by previously
studied random feature schemes. We would like to emphasize that our approach
does not stand in contrast to, but rather complements, prior work.  Indeed, the
base kernels of a compositional kernel can be non-elementary such as the
Gaussian kernel, and hence our RFS can be used in conjunction with the well-studied RFS of~\cite{RahimiRe07} for Gaussian kernels. One can also
envision a hybrid structure where the base kernels are defined, for example,
through a PSD matrix obtained by metric learning on the original input space.

\section{Background and notation}
\label{background:sec}
We start with a few notational conventions used throughout the paper.
The Hilbert spaces we consider are over the reals. This includes spaces that
are usually treated as complex Hilbert spaces. For example, for
$z=a+ib,z'=a'+ib'\in \complex$ we denote $\inner{z,z'} = \Re(z\bar z') =
aa'+bb'$ (rather than the more standard $\inner{z,z'} = z\bar z'$). Likewise,
for $\z,\z'\in \complex^q$ we denote $\inner{\z,\z'}=\sum_{i=1}^q
\inner{z_i,z'_i}$. For a measure space $(\Omega,\mu)$, $L^2(\Omega)$ denotes
the space of square integrable functions $f:\Omega\to \complex$. For $f,g\in L^2(\Omega)$ we denote
$\inner{f,g}_{L^2(\Omega)} = \int_{\Omega}\inner{f(x),g(x)}d\mu(x)$.
For all the measurable spaces we
consider, we assume that singletons are measurable. We denote
$\torus = \{z\in\complex : |z|=1\}$.

% \subsection{Kernel Spaces}

Next, we introduce the notation for kernel spaces and recap some of their
properties. A {\em kernel} is a function
$k:\cx\times\cx\to\reals$ such that for every $\x_1,\ldots,\x_m\in\cx$ the
matrix $\{k(\x_i,\x_j)\}_{i,j}$ is positive semi-definite.
We say that $k$ is
{\em $D$-bounded} if $k(\x,\x)\le D^2$ for every $\x\in\cx$. We call $k$ a
{\em normalized} kernel if $k(\x,\x) = 1$ for every $\x\in\cx$. We will always assume that kernels are normalized.
A {\em kernel space} is a Hilbert space $\ch$ of functions from
$\cx$ to $\reals$ such that for every $\x\in\cx$ the linear functional
$f\in\ch\mapsto f(\x)$ is bounded.  The following theorem describes a
one-to-one correspondence between kernels and kernel spaces.

\begin{theorem}\label{thm:ker_spaces}
For every kernel $k$ there exists a unique kernel space $\ch_k$ such that for
every $\x,\x'\in\cx$, $k(\x,\x') = \inner{k(\cdot,\x),k(\cdot,\x')}_{\ch_k}$.
Likewise, for every kernel space $\ch$ there is a kernel $k$ for which
$\ch=\ch_k$.
\end{theorem}
\noindent
The following theorem underscores
a tight connection between kernels and embeddings of $\cx$ into Hilbert spaces.
\begin{theorem}\label{thm:rkhs_embed}
A function $k:\cx \times \cx \to \reals$ is a kernel if and only if there
exists a mapping $\phi: \cx \to\ch$ to some Hilbert space for which
$k(\x,\x')=\inner{\phi(\x),\phi(\x')}_{\ch}$. In this case,
$ \ch_k = \{f_\mathbf{v} \mid \mathbf{v}\in\ch \} $
where $f_\mathbf{v}(\x) = \inner{\mathbf{v},\phi(\x)}_\ch$. Furthermore,
$\|f\|_{\ch_k} = \min\{\|\mathbf{v}\|_\ch\mid f=f_\mathbf{v}\}$ and
the minimizer is unique.
\end{theorem}
\noindent
We say that $f:[-1,1]\to \reals$ is a normalized
{\em positive semi-definite} (PSD) function if
$$f(\rho) = \sum_{i=0}^\infty a_i\rho^i \, \mbox{ where } \,
\sum_{i=0}^\infty a_i =1,\, \forall i: a_i\geq 0\,.$$
Note that $f$ is PSD if and only if $f\circ k$ is a normalized kernel for
any normalized kernel $k$.

\section{Random feature schemes}
\label{rfs:sec}
Let $\cx$ be a measurable space and let $k:\cx\times\cx\to \reals$ be a
normalized kernel.  A {\em random features scheme} (RFS) for $k$ is a pair
$(\psi,\mu)$ where $\mu$ is a probability measure on a measurable space
$\Omega$, and $\psi:\Omega\times\cx\to \complex$ is a measurable function,
such that
\[
\forall \x,\x'\in\cx,\;\;\;\;k(\x,\x') =
  \E_{\omega\sim \mu}\left[\psi(\omega,\x)\overline{\psi(\omega,\x')}\right]\,.
\]
Since the kernel is real valued, we have in this case that,
\begin{eqnarray}
k(\x,\x') &=& \Re\left(k(\x,\x')\right)
\nonumber \\
  &=& \E_{\omega\sim\mu}
    \left[\Re\left(\psi(\omega,\x)\overline{\psi(\omega,\x')}\right)\right]
\nonumber \\
  &=& \E_{\omega\sim \mu}\inner{\psi(\omega,\x),\psi(\omega,\x')} \,.
\label{eq:ker_eq_inner}
\end{eqnarray}
We often refer to $\psi$ as a random feature scheme.
We define the {\em norm} of $\psi$ as $\|\psi\| = \sup_{\omega,\x}|\psi(\omega,\x)|$. We say that $\psi$ is {\em $C$-bounded} if $\|\psi\|\le C$. As the kernels are normalized,
\eqref{eq:ker_eq_inner} implies that $\|\psi\|\ge 1$ always. In light of this, we say that an RFS $\psi$  is
{\em \nrme} if it is $1$-bounded.
Note that in this case, since the kernel is normalized, it holds that
$|\psi(\omega,\x) | =1$ for almost every $\omega$ as otherwise we would
obtain that $k(\x,\x)<1$. Hence, we can assume w.l.o.g.\ that the range of
\nrme\ RFSs is $\torus$.

\begin{remark}[From complex to real RFSs]\normalfont
While complex-valued features would simplify the analysis of random
feature schemes, it often favorable to work in practice with real-valued
features. Let $\psi(\omega,\x):=R_\omega(\x)e^{i\theta_\omega(\x)}$
be a $C$-bounded RFS for $k$. Consider the RFS
$$\psi'((\omega,b),\x):=
  \sqrt{2}R_\omega(\x)\cos\left(\theta_\omega(\x) + b\right) \,,$$
where $\omega\sim\mu$, and $b\in \left\{0,\frac{\pi}{2}\right\}$ is
distributed uniformly and independently from $\omega$. It is not difficult
to verify that $\psi'$ is $\sqrt{2}C$-bounded RFS for $k$.
\end{remark}\noindent
A random {\em feature} generated from $\psi$ is a random function
$\psi(\omega,\cdot)$ from $\cx$ to $\complex$ where $\omega\sim\mu$.
A random {\em $q$-embedding} generated from $\psi$ is the random mapping
\[
\Psi_\vomega(\x) \eqdef
\frac{\left(\psi({\omega_1},\x),\ldots , \psi({\omega_q},\x) \right)}
     {\sqrt{q}} \,,
\]
where $\omega_1,\ldots,\omega_q\sim \mu$ are i.i.d. The random
{\em $q$-kernel} corresponding to $\Psi_\vomega$ is $k_\vomega(\x,\x') =
\inner{\Psi_\vomega(\x),\Psi_\vomega(\x')}$. Likewise, the random
{\em $q$-kernel space} corresponding to $\Psi_\vomega$ is
$\ch_{k_\vomega}$. For the rest of this section, let us fix a $C$-bounded RFS $\psi$ for a
normalized kernel $k$ and a random $q$ embedding $\Psi_\vomega$.
For every $\x,\x'\in \cx$
\[
k_\vomega(\x,\x') =
  \frac{1}{q}\sum_{i=1}^q \inner{\psi({\omega_i},\x),\psi({\omega_i},\x')}
\]
is an average of $q$ independent random variables whose expectation is
$k(\x,\x')$. By Hoeffding's bound we have the following theorem.
\begin{theorem}[Kernel Approximation]\label{thm:ker_apx}
Assume that $q\ge \frac{2C^4\log\left(\frac{2}{\delta}\right)}{\epsilon^2}$,
then for every $\x, \x'\in \cx$ we have
$\Pr\left(\left| k_\vomega(\x,\x') - k(\x,\x') \right| \ge \epsilon \right)
  \le \delta$.
\end{theorem}
\noindent We next discuss approximation of functions in $\ch_k$ by functions in
$\ch_{k_\vomega}$. It would be useful to consider the embedding
\begin{equation}\label{eqn:psi-embedding}
\x\mapsto\Psi^\x \; \mbox{ where } \;
  \Psi^\x\eqdef\psi(\cdot,\x)\in L^2(\Omega) \,.
\end{equation}
From~\eqref{eq:ker_eq_inner} it holds
that for any $\x,\x'\in\cx$,
$k(\x,\x') = \inner{\Psi^\x,\Psi^{\x'}}_{L^2(\Omega)}$.
In particular, from Theorem~\ref{thm:rkhs_embed}, for every $f\in\ch_k$ there
is a unique function $\check{f}\in L^2(\Omega)$ such that
$\|\check{f}\|_{L^2(\Omega)} = \|f\|_{\ch_k}$ and for every $\x\in\cx$,
\begin{equation}\label{eq:f_x_as_inner}
f(\x) = \inner{\check{f},\Psi^\x}_{L^2(\Omega)} =
  \E_{\omega\sim\mu}\inner{\check{f}(\omega),\psi(\omega,\x)}\,.
\end{equation}
Let us denote
$f_\vomega(\x) = \frac{1}{q}\sum_{i=1}^q
  \inner{\check{f}(\omega_i),\psi(\omega_i,\x)}$.
From~\eqref{eq:f_x_as_inner} we have that
$\E_\vomega\left[f_\vomega(\x)\right] = f(\x)$.
Furthermore, for every $\x$,
the variance of $f_\vomega(\x)$ is at most
\begin{eqnarray*}
\frac{1}{q}\E_{\omega\sim\mu}
  \left|\inner{\check{f}(\omega),\psi(\omega,\x)}\right|^2
&\le &
\frac{C^2}{q}\E_{\omega\sim\mu}
  \left|\check{f}(\omega)\right|^2 
\\
&=& \frac{C^2\|f\|^2_{\ch_k}}{q}\,.
\end{eqnarray*}
An immediate consequence is the following corollary.
\begin{corollary} [Function Approximation] \label{thm:func_apx}
For all $\x\in\cx$,
$\E_\vomega|f(\x) - f_\vomega(\x)|^2 \le \frac{C^2\|f\|^2_{\ch_k}}{q}$.
\end{corollary} \noindent
As a result, if $\chi$ is a distribution on $\cx$, we have
\begin{eqnarray*}
\E_\vomega \|f-f_\vomega\|_{2,\chi}
& =   & \E_\vomega \sqrt{\E_\chi |f(\x) - f_\vomega(\x)|^2} \\
& \le & \sqrt{\E_\vomega \E_\chi |f(\x) - f_\vomega(\x)|^2} \\
& =   & \sqrt{\E_\chi \E_\vomega  |f(\x) - f_\vomega(\x)|^2}\\
  &\le &\frac{C\|f\|_{\ch_k}}{\sqrt{q}} \,.
\end{eqnarray*}
We next consider supervised learning with RFS. Let $\cy$ be a target (output)
space and let $\ell:\reals^t\times\cy\to \reals_+$ be a $\rho$-Lipschitz loss
function, i.e.\ for every $y\in\cy$,
$|\ell(y_1,y)-\ell(y_2,y)|\leq\rho|y_1-y_2|$.
Let $\cd$ be a distribution on $\cx\times\cy$. We define the loss of a
(prediction) function $\f:\cx\to \reals^t$ as
$L_\cd(\f) = \E_{(\x,y)\sim\cd}\ell(\f(\x),y)$. Let
$S = \{(\x_1,y_1),\ldots,(\x_m,y_m)\}$ denote $m$ i.i.d. examples sampled
from $\cd$.
We denote by $\ch_k^t$ the space of all functions
$\f=(f_1,\ldots,f_t):\cx\to\reals^t$ where $f_i\in \ch_k$ for every $i$.
$\ch_k^t$ is a Hilbert space with the inner product
$\inner{\f,\g}_{\ch^t_k} = \sum_{i=1}^t\inner{f_i,g_i}_{\ch_k}$.
Let $\hat{\f}$ be the function in $\ch^t_k$ that minimizes the
regularized empirical loss,
$$L_S^\lambda (\f) = \frac{1}{m}\sum_{i=1}^m\ell(\f(\x_i),y_i) +
  \lambda\| \f\|_{\ch^t_k}^2,$$ over all functions in $\ch^t_k$.
It is well established (see e.g. Corollary~13.8 in~\cite{shalev2014understanding}) that
for every $\f^\star\in \ch^t_k$,
\begin{equation}\label{eq:RLM_eq}
 \E_SL_\cd (\hat \f) \le L_\cd \left(\f^\star\right) +
  \lambda\|\f^\star\|_{\ch_k^t}^2 + \frac{2\rho^2}{\lambda m} \,.
\end{equation}
If we further assume that $\|\f^\star\|_{\ch^t_k}\le B$, for $B>0$, and set
$\lambda=\frac{\sqrt{2} \rho}{ \sqrt{m}B}$, we obtain that
\begin{equation}\label{eq:RLM_eq_2}
\E_SL_\cd (\hat \f) \le
  \inf_{\|\f^\star\|_{\ch^t_k}\le B} L_\cd \left(\f^\star\right) +
  \frac{\sqrt{8}\rho B}{ \sqrt{m}} \,.
\end{equation}
The additive term in \eqref{eq:RLM_eq_2} is optimal, up to a constant
factor. We would like to obtain similar bounds for
an algorithm that minimizes the regularized loss w.r.t.\ the embedding
$\Psi_{\vomega}$. Let $\hat{\f}_\vomega$ be the function that minimizes,
\begin{equation}\label{eq:frisk}
L_S^\lambda(\f) =
  \frac{1}{m}\sum_{i=1}^m \ell(\f(\x_i),y_i) +
  \lambda\| \f\|_{\ch^t_{k_\vomega}}^2 \,,
\end{equation}
over all functions in $\ch^t_{k_\vomega}$.
Note that in most settings $\hat{\f}_\vomega$ can be found efficiently by
defining a matrix $V \in \complex^{t\times q}$ whose $i$'th row is $\bv_i$,
and rewriting $\hat{\f}_\vomega$ as,
\[
\hat{\f}_\vomega (\x) =
\left(
  \inner{\bv_1,\Psi_\vomega(\x)},\ldots,\inner{\bv_t,\Psi_\vomega (\x)}
\right) \eqdef {V \Psi_\vomega (\x)} \,.
\]
We now can recast the empirical risk minimization of~\eqref{eq:frisk} as,
$$
L_S^\lambda(V) =
  \frac{1}{m}\sum_{i=1}^m\ell(V\Psi_\vomega (\x_i),y_i) +
    \lambda\| V\|_{\mathrm{\footnotesize F}}^2 \,.
$$

\begin{theorem}[Learning with RFS]\label{thm:rfs_learn}
For every $\f^\star\in\ch^t_k$,
\begin{equation}\label{eq:RLM_embed_eq}
  \E_\vomega\E_SL_\cd (\hat \f_\vomega) \le
    L_\cd \left(\f^\star\right) + \lambda\|\f^\star\|_{\ch^t_{k}}^2 +
    \frac{2\rho^2 C^2}{\lambda m} +
    \frac{\rho \|\f^\star\|_{\ch^t_{k}} C }{\sqrt{q}} \,.
\end{equation}
If we additionally impose the constraint ${\|\f^\star\|_{\ch^t_k}\le B}$
for $B>0$ and set $\lambda=\frac{\sqrt{2} \rho C}{\sqrt{m} B}$ we have,
\begin{equation}\label{eq:RLM_embed_eq_2}
\E_\vomega\E_SL_\cd (\hat \f_\vomega) \le
  \inf_{\|\f^\star\|_{\ch^t_k}\le B} L_\cd \left(\f^\star\right) +
  \frac{\sqrt{8}\rho BC}{ \sqrt{m}} + \frac{\rho B C }{\sqrt{q}} \,.
\end{equation}
\end{theorem}
We note that for \nrme\ RFS (i.e.\ when $C=1$), if the number of random
features is proportional to the number of examples, then the error terms in
the bounds (\ref{eq:RLM_embed_eq_2}) and \eqref{eq:RLM_eq_2} are the same up
to a multiplicative factor. Since the error term in~\eqref{eq:RLM_eq_2} is
optimal up to a constant factor, we get that the same holds true
for~\eqref{eq:RLM_embed_eq_2}.

\begin{proof}
For simplicity, we analyze the case $t=1$. Since $k_\vomega$ is $C$-bonded,
we have from~\eqref{eq:RLM_eq} that,
\[
  \E_SL_\cd \left(f\right) \le L_\cd \left(f_\vomega^\star\right) +
    \lambda\|f_\vomega^\star\|_{\ch_{k_\vomega}}^2 +
    \frac{2\rho^2C^2}{\lambda m}\,.
\]
Hence, it is enough to show that $\E_\vomega
\|f_\vomega^\star\|_{\ch_{k_\vomega}}^2 \le \|f^\star\|_{\ch_k}^2$ and
$\E_\vomega L_\cd \left(f_\vomega^\star\right) \le L_\cd \left(f^\star\right)
  + \frac{\rho \|f^\star\|_{\ch_k} C}{\sqrt{q}}$. Indeed, since
$$f^\star_\vomega(\x) =
\left\langle\frac{(\check{f}^\star(\omega_1),\ldots,\check{f}^\star(\omega_q))}
  {\sqrt{q}},\Psi_\vomega(\x)\right\rangle,$$
we have, by Theorem \ref{thm:rkhs_embed},
\begin{eqnarray*}
\E_\vomega \|f_\vomega^\star\|_{\ch_{k_\vomega}}^2 &\le&
  \E_\vomega
  \left[\frac{\sum_{i=1}^q\left|\check{f}^\star(\omega_i)\right|^2} {q}\right]
  \\
  &=& \frac{1}{q}\sum_{i=1}^q\E_{\omega_i}
    \left|\check{f}^\star(\omega_i)\right|^2
    \\
  &=& \|f^\star\|^2_{\ch_k},
\end{eqnarray*}
and similarly,
\begin{eqnarray*}
\E_\vomega L_\cd \left(f_\vomega^\star\right)
  =
  \E_{\vomega}\E_{\cd}l(f_\vomega^\star(\x),y)
  =
  \E_{\cd}\E_{\vomega}l(f_\vomega^\star(\x),y) \,.
\end{eqnarray*}
Now, from the $\rho$-Lipschitzness of $\ell$ and Theorem~(\ref{thm:func_apx})
we obtain,
\begin{eqnarray*}
\E_{\vomega}\ell(f_\vomega^\star(\x),y) &\le& \ell(f^\star(\x),y) +
  \rho \E_{\vomega} |f^\star(\x)-f_\vomega^\star(\x)|
\\
&\le& \ell(f^\star(\x),y) + \rho \sqrt{\E_{\vomega} |f^\star(\x)-f_\vomega^\star(\x)|^2}
\\
&\le& \ell(f^\star(\x),y) + \frac{\rho \|f^\star\|_{\ch_k} C}{\sqrt{q}},
\end{eqnarray*}
concluding the proof.
\end{proof}

\iffalse
\section{Random features for compositional kernels} \label{rfs4ck:sec}
%
In the section we discuss the algorithmic aspects for building random feature
schemes from compositional kernels. As the name implies, a compositional
kernel consists of a succession of compositions of prior constructed kernels.
The start of the composition process uses basic kernels and their
corresponding RFS. We first describe constructions of RFS for commonly used
simple spaces.
\fi

\section{Random feature schemes for basic spaces}
\label{rfs4basic:sec}
In order to apply Theorem \ref{thm:rfs_learn}, we need to control the boundedness of the
generated features. Consider the RFS generation procedure, given in
Algorithm~\ref{alg:skel_rfs}, which employs multiplications of features
generated from basic RFSs. If each basic RFS is $C$-bounded, then every feature
that is a multiplication of $t$ basic features is $C^t$-bounded. In light of
this, we would like have RFSs for the basic spaces whose norm is as small as
possible.  The best we can hope for is {\em \nrme}\ RFSs---namely, RFSs with norm of 1.
We first describe such RFSs
for several kernels including the Gaussian kernel on $\reals^d$, and the
standard inner product on $\mathbb S^0$ and $\mathbb S^1$.  Then, we discuss
the standard inner product on $\mathbb S^{d-1}$ for $d\ge 3$. In this case, we
show that the smallest possible norm for an RFS is $\sqrt{d/2}$.
Hence, if the basic spaces are $\mathbb S^{d-1}$ for $d\ge 3$, one might
prefer to use other kernels such as the Gaussian kernel.

\begin{example}[Binary coordinates]\label{exam:binary}\normalfont
Let $\cx = \{\pm 1\}$ and $k(x,x')=xx'$. In this case the deterministic
identity RFS $\psi(\omega,x) = x$ is \nrme.
\end{example}

\begin{example}[One dimensional sphere]\label{exam:one_dim}\normalfont
Let $\cx = \mathbb{T}$ and $k(z,z')=\inner{z,z'}$. Let
$\psi(\omega,\z) = \z^\omega$, where $\omega$ is either $-1$ or $+1$ with
equal probability. Then, $\psi$ is a \nrme\ RFS since
\[
  \E_{\omega\sim\mu} \psi(\omega,z)\overline{\psi(\omega,z')} =
    \frac{z\overline{z'}+\overline{z}z'}{2} = \inner{z,z'} \,.
\]
\end{example}

\begin{example}[Gaussian kernel]\label{exam:gauss}\normalfont
Let $\cx = \mathbb R^d$ and $k(\x,\x')=e^{-\frac{a^2\|\x-\x'\|^2}{2}}$,
where $a>0$ is a constant. The Gaussian RFS is
$\psi(\vomega,\x) = e^{ia\inner{\vomega,\x}}$, where $\vomega \in \mathbb R^d$
is the standard normal distribution. Then, $\psi$ is a \nrme\ RFS, as implied
by~\cite{RahimiRe07}.
\end{example}

\begin{example}[Categorical coordinates]\label{exam:categorical}\normalfont
Let $\cx = [n]$ and define $k(x,x')=1$ if $x=x'$ and $0$ otherwise.
In this case
$\psi(\omega,x) = e^{\frac{i\omega x}{2 \pi n}}$, where $\omega$ is distributed
uniformly over $[n]$, is a \nrme\ RFS since,
\[
\E_{\omega\sim\mu} \psi(\omega,x)\overline{\psi(\omega,x')} =
  \E_\omega e^{\frac{i\omega (x-x')}{2 \pi n}} = \begin{cases}
  1 & x=x' \\ 0 & x\ne x'
\end{cases} \,.
\]
\end{example}

\iffalse
\todonow{needs editing\dots
-----------------------------------------------------------------------------}
\begin{remark*}[Translation Invariant Kernels]\normalfont
%
Let $G$ be a locally compact abelian group. By Bochner's theorem~\cite{Rudin94},
every continuous and normalized kernel $k:G\times G\to\reals$ that satisfies
$k(\x,\x') = k(\x+\z,\x'+\z)$ has a \nrme\ RFS.
We remark that a certain converse of Bochner theorem is true. Indeed,
suppose that $\psi:\Omega\times \cx\to \mathbb \torus$ is a \nrme\ RFS for
the normalized kernel $k:\cx \times \cx\to\reals$. We claim that $k$ is
necessarily a restriction of a continuous, normalized and translation
invariant kernel $k'$ on an Abelian group $G$ (that is not necessarily
locally-compact). Namely, there is an embedding $\Phi:\cx\to G$ such that
$k(\x,\x') = k'\left(\Phi(\x),\Phi(\x')\right)$.  To see this, let $G =
\{f:\Omega\to \mathbb T\} \subset L^2(\Omega)$. $G$ is a metric Abelian
group w.r.t. the pointwise function multiplication and the metric inherited
from $L^2(\Omega)$. Furthermore, $k'(f,g):= \inner{f,g}_{L^2(\Omega)}$ is a
continuous, normalized and translation invariant kernel $k'$ on $G$.
Finally, from \eqref{eq:f_x_as_inner} we have
$k(\x,\x') = k'\left(\Phi(\x),\Phi(\x')\right)$ for the embedding
$\Phi(\x)(\omega) = \Psi^x(\omega)$.
%
\end{remark*}
\todonow{
\centerline
{-----------------------------------------------------------------------------}}
\fi

Examples \ref{exam:binary} and \ref{exam:one_dim} show that the standard inner
product on the sphere in one and two dimensions admits a \nrme\ RFS. We next
examine the sphere $\sphere^{d-1}$ for $d\geq3$. In this case, we show a
construction of a $\sqrt{{d}/{2}}$-bounded RFS. Furthermore, we show that
it is the best attainable bound. Namely, any RFS for
$\sphere^{d-1}$ will necessarily have a norm of at least $\sqrt{d/2}$. In
particular, there does not exist a \nrme\ RFS when $d\ge 3$.

\begin{example}[$\sphere^{d-1}$ for $d\ge 3$]\label{exam:sphere}\normalfont
Let $\mu$ be the uniform distribution on $\Omega = [d]\times \{-1,+1\}$.
Define $\psi:\Omega\times\sphere^{d-1}\to\complex$ for $\omega=(j,b)$ as
$\psi(\omega,\x) = \sqrt{{d}/{2}}(x_j + i b x_{j+1})$,
where we use the convention $x_{d+1}:=x_1$. Now, $\psi$ is
a $\sqrt{{d}/{2}}$-bounded RFS as,
\begin{eqnarray*}
&& \E_{(j,b)\sim\mu} \psi( (j,b),\x)\overline{\psi((j,b),\x')}\\
  &&= \frac{\frac{d}{2}\sum_{j=1}^d \left[
      (x_j+ix_{j+1})(x'_j-ix'_{j+1})+(x_j-ix_{j+1})(x'_j+ix'_{j+1})
      \right]}{2d} \\
  &&= \frac{\sum_{j=1}^d \left[x_jx'_j+x_{j+1}x'_{j+1}\right]}{2} \\
  &&= \inner{\x,\x'} \,.
\end{eqnarray*}
\end{example}
\noindent
We find it instructive to compare the RFS above to the following
$\sqrt{d}$-bounded RFSs.
\begin{example}
% [$\sqrt{d}$-bounded RFS on $\sphere^{d-1}$]
\label{exam:sphere2}\normalfont
\iffalse
Let $\mu$ be the uniform distribution on $\Omega = [d]$ and define
$\psi:\Omega\times\sphere^{d-1}\to\reals$ by $\psi_i(\x) = \sqrt{d}x_i$. We have
  \[
  \E_{i\sim\mu} \psi_i(\x)\psi_i(\x') = \frac{\sum_{i=1}^d \sqrt{d}x_i\sqrt{d}x'_i}{d} = \inner{\x,\x'}
  \]
\fi
Let $\mu$ be the uniform distribution on $\Omega = \sphere^{d-1}$ and define
$\psi:\Omega\times\sphere^{d-1}\to\reals$ as $\psi(\w,\x) =
\sqrt{d}\inner{\w,\x}$. We get that,
$$
\E_{\w\sim\mu}\left[\psi(\w,\x)\,\psi(\w,\x')\right]
  \, = \, d\E_{\w\sim\mu} \inner{\w,\x}\inner{\w,\x'} \, = \,
    d\,\inner{\x, W \x'} ~,
$$
where $W_{i,j} = \E_{\w\sim\mu}\left[w_i w_j\right]$. Since $\w$ is
distributed uniformly on $\sphere^{d-1}$, $\E\left[w_i w_j\right]= 0$ for
$i\neq j$ and $\E\left[w_i^2\right] = 1 / d$. Thus, we have
$W = (1/d)\mathbb I$ and therefore
$\E_{\w\sim\mu}\left[\psi(\w,\x)\,\psi(\w,\x')\right]=\inner{\x,\x'}$. A
similar result still holds when $\psi(\omega,\x) = \sqrt{d}\,x_\omega$ where
$\omega\in [d]$ is distributed uniformly.
\end{example}

\noindent
To conclude the section, we prove that $\sqrt{{d}/{2}}$-boundedness is
optimal for RFS on $\sphere^{d-1}$.

\begin{theorem}\label{thm:sphere_lower_bound}
Let $d\ge 1$ and $\epsilon > 0$. There does not exist
a $(\sqrt{{d}/{2}}-\epsilon)$-bounded RFS for the kernel
$k(\x,\x')=\inner{\x,\x'}$ on $\sphere^{d-1}$.
\end{theorem}
\noindent Before proving the theorem, we need the following 
lemmas.
\begin{lemma}\label{fact}
Let $\z\in\complex^d$. There exists $\ba\in\sphere^{d-1}$ such that
$\left|\sum_{j} a_j z_j\right|^2 \ge \frac{1}{2}\|\z\|^2$.
\end{lemma}

\begin{proof}
Let us write $\z=\valpha +  i\vbeta$ where $\valpha,\vbeta\in\reals^d$. We
thus have $\|\z\|^2 = \|\valpha\|^2 + \|\vbeta\|^2$. We can further assume
that $ \|\valpha\|^2 \geq \frac{1}{2}\|\z\|^2$ as otherwise we can replace
$\z$ with $i\z$. Let us define $\ba$ as $\valpha/\|\valpha\|$. We now obtain
that,
\[
\left|\displaystyle \sum_{j}a_jz_j\right|^2 \ge \inner{\ba,\valpha}^2 =
  \|\valpha\|^2 \geq \frac{1}{2}\|\z\|^2 \,,
\]
which concludes the proof.
\end{proof}

\begin{lemma}\label{lem:closed_to_lin_comb}
Let $(\psi,\mu)$ be an RFS for the kernel $k(\x,\x')=\inner{\x,\x'}$ on
$\sphere^{d-1}$ and let $\mathbf a \in\sphere^{d-1}$. Then, for almost all
$\omega$ we have
$\psi({\omega},\mathbf a) = \sum_{j=1}^d a_j\psi({\omega},\e_j)$.
\end{lemma}
\begin{proof}
  Let us examine the difference between
$\Psi^\ba\eqdef\psi(\cdot,\ba)$
and $\sum_ja_j\Psi^{\e_j}\eqdef\sum_ja_j\psi(\cdot,\e_j)$,
\begin{eqnarray*}
&& \left\|\Psi^\ba - \sum_{i=1}^d a_i\Psi^{\e_i}\right\|_{L^2(\Omega)}^2\\
  &&= \inner{\Psi^\ba,\Psi^\ba}_{L^2} +
      \sum_{i,j=1}^d a_ia_j\inner{\Psi^{\e_i},\Psi^{\e_j}}_{L^2} -
      2 \sum_{i=1}^d a_i\inner{\Psi^{\ba},\Psi^{\e_i}}_{L^2}
  \\
  &&= \inner{\ba,\ba} + \sum_{i,j=1}^d a_ia_j\inner{\e_i,\e_j} -2 \sum_{i=1}^d a_i\inner{\ba,\e_i}
  \\
  &&= \left\|\ba - \sum_{i=1}^d a_i\e_i\right\|^2 = 0.
\end{eqnarray*}
\end{proof}

\begin{proof}[Proof of Theorem~\ref{thm:sphere_lower_bound}]
Let $\psi:\sphere^{d-1}\times\Omega\to \complex$ be an RFS for $k(\x,\x') =
\inner{\x,\x'}$ and let $\epsilon>0$. We next show that $\psi$ is not
$\sqrt{{d}/{2}-\epsilon}$\,\,-\,bounded.
Let $A\subset \sphere^{d-1}$ be a dense and countable set. From
Lemma~\eqref{lem:closed_to_lin_comb} and the fact that sets of measure
zero are closed under countable union, it follows that for almost every
$\omega$ and all $\ba\in A$ we have
$\psi({\omega},\mathbf a) = \sum_{i=1}^d a_i\psi({\omega},\e_i)$.
Using the linearity of expectation we know that,
%\scriptsize
\[
\E_{\omega\sim\mu}\sum_{i=1}^d|\psi(\omega,\e_i)|^2 =
  \sum_{i=1}^d\E_{\omega\sim\mu}|\psi(\omega,\e_i)|^2 =
  \sum_{i=1}^d \inner{\e_i,\e_i} = d \,.
\]
%\normalsize
Hence, with a non-zero probability we get,
\begin{equation}\label{eq:1}
\sum_{i=1}^d|\psi(\omega,\e_i)|^2>d-\epsilon,
\end{equation}
and,
\begin{equation} \label{eq:1a}
\forall \ba\in A,\;\psi({\omega},\mathbf a) =
  \sum_{i=1}^d a_i\psi({\omega},\e_i) \,.
\end{equation}
Let us now fix $\omega$ for which \eqref{eq:1} holds.
From Lemma~\eqref{fact} there exists $\tilde{\ba}\in \sphere^{d-1}$ such that
$\left|\sum_{i} \tilde{a}_i\psi(\omega,\e_i)\right|^2
  \ge\frac{d-\epsilon}{2}\,.$
Since $A$ is dense in $\sphere^{d-1}$ there is a vector $\ba\in A$ for which
$\left|\sum_{i}a_i\psi(\omega,\e_i)\right|^2\ge\frac{d}{2}-\epsilon$.
Finally, from~\eqref{eq:1} it follows that
$|\psi(\omega,\ba)|^2\ge \frac{d}{2}-\epsilon$.
\end{proof}

\section{Compositional random feature schemes}
\label{comp_rfs:sec}
Compositional kernels are obtained by sequentially multiplying and averaging
kernels. Hence, it will be useful to have RFSs for multiplications and
averages of kernels. The proofs of Lemmas \ref{lem:rfs_avg} and
\ref{lem:rfs_mul} below are direct consequences of properties of kernel
spaces and RFSs and thus omitted.
\begin{lemma}\label{lem:rfs_avg}
Let $(\psi^1,\mu^1),(\psi^2,\mu^2),\ldots$ be RFSs for the kernels
$k^1,k^2,\ldots$ and let $(\alpha_i)_{i=1}^\infty$ be a sequence of
non-negative numbers that sum to one. Then, the following procedure
defines an RFS for the kernel $k(\x,\x')=\sum_{i=1}^n\alpha_i k^i(\x,\x')$.
\begin{enumerate}
\item Sample $i$ with probability $\alpha_i$
\item Choose $\omega\sim\mu^i$
\item Generate the feature $\x\mapsto \psi^i_{\omega}(\x)$
\end{enumerate}
\end{lemma}
\begin{lemma}\label{lem:rfs_mul}
Let $(\psi^1,\mu^1),\ldots,(\psi^n,\mu^n)$ be RFSs for the kernels
$k^1,\ldots,k^n$. The following scheme is an RFS for the kernel
$k(\x,\x')=\prod_{i=1}^n k^i(\x,\x')$. Sample
$\omega_1,\ldots,\omega_n\sim\mu^1\times\ldots\times\mu^n$ and generate the
feature $\x\mapsto \prod_{i=1}^{n}\psi^i_{\omega_i}(\x)$.
\end{lemma}

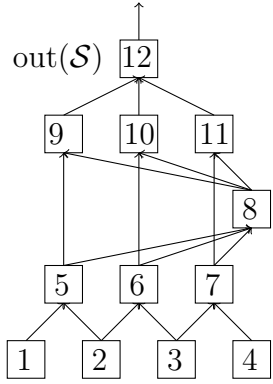
\begin{figure}[t]
\newcounter{nodenum}[nodenum]
\setcounter{nodenum}{1}
\begin{center}
\begin{tikzpicture}
\foreach \i in {1,...,4}
{
	\draw (\i -7,0) rectangle (\i-7+0.5,0.5);
	\node[text width=0.75cm] at (\i-7+0.5,0.25) {$\thenodenum$};
	\addtocounter{nodenum}{1}
}
\foreach \i in {1,...,3}
{
	\draw (\i-7 + 0.5 ,1) rectangle (\i-7+1,1.5);
	\draw [->] (\i-7 +0.25  ,0.5) -- (\i-7 + 0.75 ,1);
	\draw [->] (\i-7 +1.25  ,0.5) -- (\i-7 + 0.75 ,1);
	\draw [->] (\i-7 + 0.75 ,1.5) -- (-2.75 ,2);
	\node[text width=0.75cm] at (\i-7+1.0,1.25) {\thenodenum};
	\addtocounter{nodenum}{1}
}
\draw (-3 ,2) rectangle (-2.5,2.5);
\node[text width=0.25cm] at (-2.75, 2.25) {\thenodenum};
\addtocounter{nodenum}{1}
\foreach \i in {1,...,3}
{
	\draw (\i-7 + 0.5 ,3) rectangle (\i-7+1,3.5);
	\node[text width=0.5cm] at (\i-7+0.8, 3.25) {\thenodenum};
	\addtocounter{nodenum}{1}
	\draw [->] (\i-7 +0.75, 1.5) -- (\i-7 +0.75,3);
	\draw [->] (-2.75, 2.5) -- (\i-7 +0.75,3);
	\draw [->] (\i-7 + 0.75 ,3.5) -- (-4.25 ,4);
}
\draw (-4.5 ,4) rectangle (-4,4.5);
\node[text width=3.25cm] at (-4.3,4.25) {$\out(\cs) \;\; \thenodenum$};
\draw [->] (-4.25 ,4.5) -- (-4.25 ,5);
\end{tikzpicture}
\caption{An illustration of a computation skeleton.\label{skeleton:fig}}
\end{center}
\end{figure}
\paragraph{Random feature schemes from computation skeletons.}
We next describe and analyze the case where the compositional kernel
is defined recursively using a concrete computation graph defined below. Let
$\cx_1,\ldots,\cx_n$ be measurable spaces with corresponding normalized
kernels $k^1,\ldots,k^n$ and RFSs $\psi^1,\ldots,\psi^n$. We refer to these
spaces, kernels, and RFS as {\em base} spaces, kernels and RFSs. We also
denote $\cx = \cx_1\times \ldots\times\cx_n$. The base spaces (and
correspondingly kernels, and RFSs) often adhere to a simple form.  For
example, for real-valued input, feature $i$ is represented as
$\cx_i = \mathbb{T}$,
where $k^i(z,z') = \inner{z,z'}$, $\psi_{\omega}^i(z) = z^\omega$,
and $\omega$ is distributed uniformly in $\{\pm 1\}$.

We next discuss the procedure for generating compositional RFSs using a
structure termed {\em computation skeleton}, or {\em skeleton} for short. A
skeleton $\cs$ is a DAG with $m:=|\cs|$ nodes. $\cs$ has a single node whose
out degree is zero, termed the {\em output} node and denoted $\out(\cs)$, see
Figure~\ref{skeleton:fig} for an illustration. The nodes indexed $1$ through $n$
are input nodes, each of which is associated with a base space. We refer to
non-input nodes as internal nodes. Thus, the indices of internal nodes are in
$\{n+1,\dots,m\}$. An internal node $v$ is associated with a PSD function
(called a conjugate activation function~\cite{DanielyFrSi16}[Sec.~5]),
$\hat{\sigma}_v(\rho) = \sum_{i=0}^\infty a^v_i\rho^i$. For every node $v$
we denote by $\cs_v$ the subgraph of $\cs$ rooted at $v$. This sub-graph
defines a compositional kernel through all the nodes nodes with a directed
path to $v$. By definition it holds that $\out(\cs_v) = v$ and
$\cs_{\out(\cs)}=\cs$. We denote by $\In(v)$ the set of nodes with a directed
edge into $v$.  Each skeleton defines a kernel
$k_\cs:\cx\times\cx\to \reals$ according to the following recurrence,
\begin{equation*}
k_{\cs}(\x,\x') =\begin{cases}
k_v(\x,\x') & v \in [n] \\
\hat{\sigma}_{v}\!\left(\frac{\sum_{u\in \In\left(v\right)}k_{\cs(u)}(\x,\x')}
	{\left|\In\left(v\right)\right|}\right) &
	v \not\in [n]
\end{cases} \hspace{0.25cm} \mbox{ for } v = \out(\cs) \;.
\end{equation*}
In Figure~\ref{alg:skel_rfs} we give the pseudocode describing the RFS for the
kernel $k_\cs$. We call the routine $\rfss$ as a shorthand for a Random Feature
Scheme for a Skeleton. The correctness of the algorithm is a direct
consequence of Lemmas~\ref{lem:rfs_avg}~and~\ref{lem:rfs_mul}.
\begin{algorithm}[H]
\begin{minipage}{0.6\textwidth}
	\centering
  \caption{$\rfss\!\left(\cs\right)$}
  \begin{algorithmic}\label{alg:skel_rfs}
    \STATE Let $v=\out(\cs)$
    \IF {$v\in[n]$}
			\STATE Return $\x\mapsto\psi^v(\x)$
    \ELSE
    \STATE Sample $l\in \{0,1,2,\ldots,\}$ according to
		$\displaystyle (a^v_i)^\infty_{i=0}$
    \FOR {$j=1,\ldots,l$}
    \STATE Choose $u\in\mathrm{in}(v)$ at random
    \STATE Call
			$\rfss(\cs_u)$ and get
			$ \x\mapsto\psi_{\omega_j}(\x)$
    \ENDFOR
    \STATE Return $\x\mapsto\prod_{j=1}^l \psi_{\omega_j}(\x)$
    \ENDIF
  \end{algorithmic}
\end{minipage}
\end{algorithm}

We next present a simple running time analysis of Algorithm~\ref{alg:skel_rfs}
and the sparsity of the generated random features. Note that a compositional
random feature is a multiplication of base random features. Thus the amortized
time it takes to generate a compositional random feature and its sparsity
amount to the expected number of recursive calls made by $\rfss$. For a given
node $v$ the expected number of recursive calls emanating from $v$ is,
$$
\E_{l\sim (a^v_j)}\!\!\left[\,l\,\right] =
	\sum_{j=0}^\infty j \, a_j = \hat{\sigma}'(1) \,.
$$
We now define the {\em complexity} of a skeleton as,
\begin{equation}\label{eq:skel_comp}
\cc(\cs) =\begin{cases}
1 & \out(\cs)\in[n] \\
\hat{\sigma}'_{v}(1)
	\frac
		{\sum_{u\in \In\left(v\right)}\cc\left(\cs(u)\right)}
		{\left|\In\left(v\right)\right|} & \mbox{otherwise}
\end{cases} \,.
\end{equation}
It is immediate to verify that $\cc(\cs)$ is the expected value of the number
of recursive calls and the sparsity of a random feature. When all conjugate
activations are the same \eqref{eq:skel_comp} implies that \[ \cc(\cs) \le
\left(\hat{\sigma}'(1)\right)^{\depth(\cs)} \,,\] where equality holds when
the skeleton is layered. For activations such as ReLU, $\sigma(x) =
\max(0,x)$, and exponential, $\sigma(x) = e^x$, we have $\hat{\sigma}'(1)=1$, and thus
$\cc(\cs)=1$. Hence, it takes constant time in expectation to generate a
random feature, which in turn has a constant number of multiplications of base
random features.  For example, let us assume that the basic spaces are
$\sphere^{1}$ with the standard inner product, and that the skeleton has a
single non-input node, equipped with the exponential dual activation
$\hat{\sigma}(\rho)=e^\rho$. In this case, the resulting kernel is the
Gaussian kernel and $\cc(\cs)=1$. Therefore, the computational cost of
storing and evaluating each feature is constant. This is in contrast to
the Rahimi and Recht scheme \cite{RahimiRe07}, in which the cost is linear in $n$.

\section{Empirical Evaluation} \label{sec:concentration}

\paragraph{Accuracy of kernel approximation.}
\newcommand\cat[1]{\textbf{#1}}

We empirically evaluated the kernel approximation of our random feature
scheme under kernels of varying structure and depth. Our concern is
efficiency of random features: how does the quality of approximation
fare in response to increasing the target dimension of the feature
map? Theorem~\ref{thm:ker_apx} already establishes an upper bound for
the approximation error (in high probability), but overlooks a few
practical advantages of our construction that are illustrated in the
experiments that follow. Primarily, when one repeatedly executes
Algorithm~\ref{alg:skel_rfs} in order to generate features, duplicates may
arise. It is straightforward to merge them and hence afford to
generate more under the target feature budget. We use the CIFAR-10
dataset for visual object recognition
\citep{KrizhevskySuHi09}.
% \footnote{\url{https://www.cs.toronto.edu/~kriz/cifar.html}}

We considered two kernel structures, one \emph{shallow} and another
\emph{deep}. Figure~\ref{fig:conc-skels} caricatures both. For visual
clarity, it oversimplifies convolutions and the original image to
one-dimensional objects, considers only five input pixels, and
disregards the true size and stride of convolutions used.

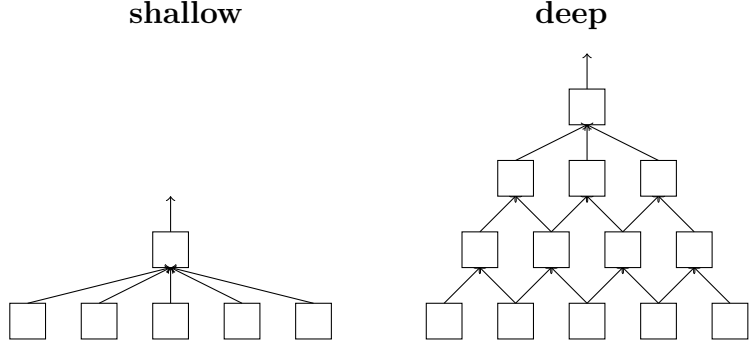
\begin{figure}[t]
  \begin{center}
  \begin{tabular}{cc}
  \textbf{shallow} & \textbf{deep} \\[.6em]
  \begin{tikzpicture}[scale=.95]
    \foreach \i in {1,...,5} {
      \draw (\i,0) rectangle (\i+.5,.5);
      \draw [->] (\i+.25,.5) -- (3.25,1);
    }
    \draw (3.0,1) rectangle (3.5,1.5);
    \draw [->] (3.25,1.5) -- (3.25,2);
  \end{tikzpicture}
  ~~ & ~~
  \begin{tikzpicture}[scale=.95]
    \foreach \i in {1,...,5} {
      \draw (\i,0) rectangle (\i+.5,.5);
    }
    \foreach \i in {1,...,4} {
      \draw [->] (\i+.25,.5) -- (\i+.75,1);
      \draw [->] (\i+1.25,.5) -- (\i+.75,1);
    }
    \foreach \i in {1,...,4} {
      \draw (\i+.5,1) rectangle (\i+1,1.5);
    }
    \foreach \i in {1,...,3} {
      \draw [->] (\i+.5+.25,1.5) -- (\i+1+.25,2.0);
      \draw [->] (\i+.5+1.25,1.5) -- (\i+.5+.75,2.0);
    }
    \foreach \i in {1,...,3} {
      \draw (\i+1,2) rectangle (\i+1+.5,2.5);
      \draw [->] (\i+1+.25,2.5) -- (3.25,3);
    }
    \draw (3.0,3) rectangle (3.5,3.5);
    \draw [->] (3.25,3.5) -- (3.25,4);
  \end{tikzpicture}
  \end{tabular}
  \caption{Simple illustration of the skeleton structures
    corresponding to the shallow (left) and deep (right) settings in
    evaluating the approximation (Figure~\ref{fig:concentration}). The
    deep setting considers two layers of convolutions followed by one
    that is fully connected.}
  \label{fig:conc-skels}
  \end{center}
\end{figure}

Following \citet{DanielyFrSi16}, for a scalar function $\sigma : \mathbb{R}
\to \mathbb R$ termed an activation, let $\hat\sigma(\rho) = \mathbb E_{(X,Y)
\sim N_\rho} [\sigma(X)\sigma(Y)]$ be its conjugate activation (shown in the
original to be a PSD function). Our shallow structure is simply a Gaussian
(RBF) kernel with scale $0.5$. Equivalently, again borrowing the terminology
of \citet{DanielyFrSi16}, it corresponds to a single-layer fully-connected
skeleton, having a single internal node $v$ to which all input nodes point.
The node $v$ is labeled with a conjugate activation $\hat\sigma_v(\rho) =
\exp((\rho-1)/4)$.

The deep structure comprises a layer of 5x5 convolutions at stride 2
with a conjugate activation of $\hat\sigma_1(\rho) =
\exp((\rho-1)/4)$, followed by a layer of 4x4 convolutions at stride 2
with the conjugate activation $\hat\sigma_2$ corresponding to the ReLU
activation $\sigma_2(t) = \max\{0,t\}$, followed by a fully-connected
layer again with the ReLU's conjugate activation.

In each setting, we compare to a natural baseline built (in part, where
possible) on the scheme of \citet{RahimiRe07} for Gaussian kernels. In the
shallow setting, doing so is straightforward, as their scheme applies
directly. As their scheme is not defined for compositional kernels, our
baseline in the deep setting is a hybrid construction. A single random
features is generated according to the recursive procedure of
Algorithm~\ref{alg:skel_rfs}, until an internal node is reached \emph{in the
bottom-most convolutional layer}. Each such node corresponds to a Gaussian
kernel, and so we apply the scheme \citet{RahimiRe07} to approximate the
kernel of that node.

For each configuration of true compositional kernel and feature
budget, we repeat the following ten times: draw a batch of 128 data
points at random (each center-cropped to 24x24 image pixels), generate
a randomized feature map, and compute $128^2$ kernel evaluations. The
result is $10 \cdot 128^2$ corresponding evaluations of the true
kernel $k$ and an empirical kernel $\hat k$. We compare error measures
and correlation between the two kernels using (i) RFS inner products
as the empirical kernel and (ii) inner products from the baseline
feature map. Figure~\ref{fig:concentration} plots the comparison.

%\twocolumn[
\begin{figure*}[t]
  %\hspace{-.7em}
  \begin{center}
  \includegraphics[scale=.525]{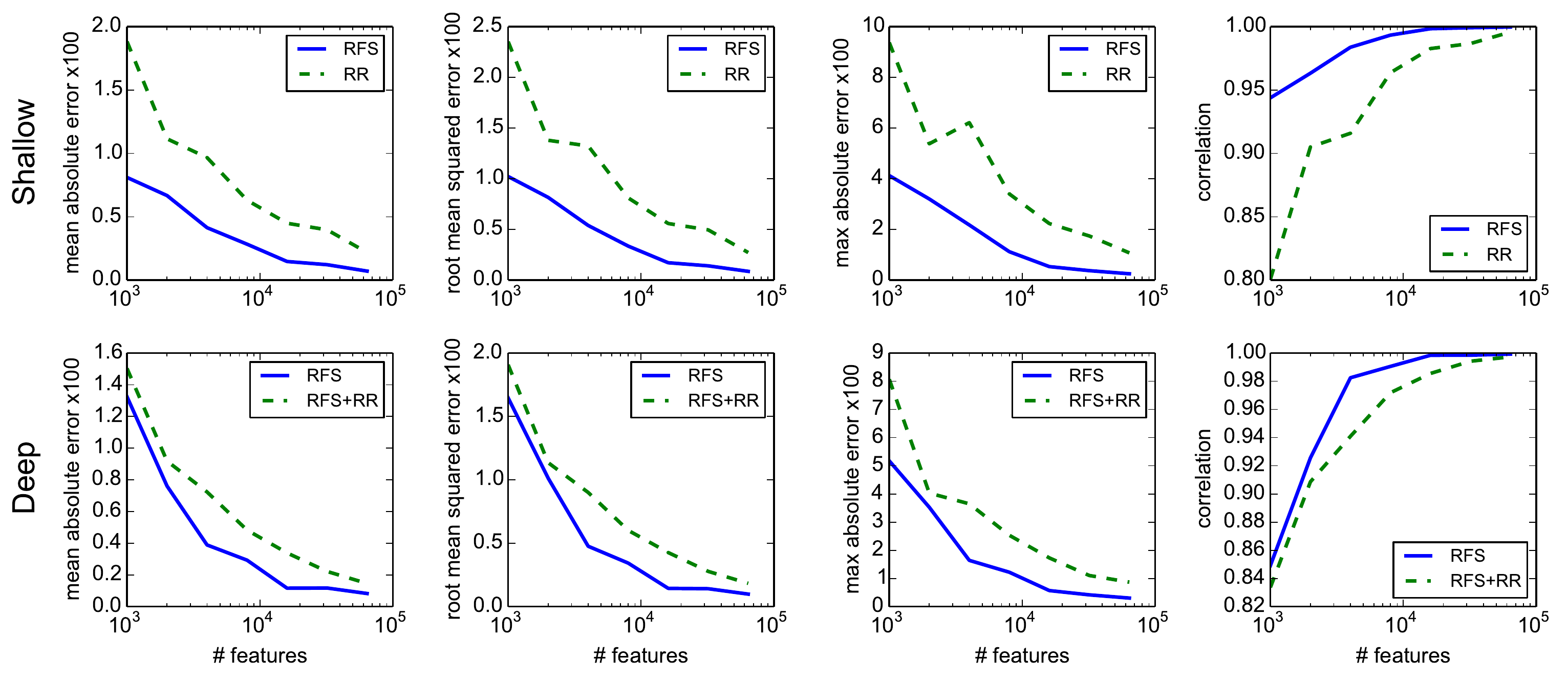}
  \caption{Comparisons of the mean-absolute, root-mean-squared, and
    maximum differences, as well as the correlation, between an
    empirical kernel and the ground truth kernel. RFS denotes our
    construction. In the shallow setting, the baseline (RR) is the
    scheme of \citet{RahimiRe07}
    for Gaussian kernels. In the deep
    setting, the baseline (RFS+RR) uses the recursive sampling scheme
    of RFS to arrive at a first-level convolution, and then mimics RR
    on the convolution patch.  The RFS approximation is dominant in
    every comparison.}
  \label{fig:concentration}
  \end{center}
\end{figure*}
%]

\paragraph{Structural effects and relation to neural networks.}
In connection to deep neural networks, we experiment with the effect of
compositional kernel architecture on learning.  We explore two questions: (i)
Is a convolutional structure effective in a classifier built on random
features?  (ii) What is the relation between learning a compositional kernel
and a neural network corresponding to its skeleton?  The second question is
motivated by the connection that \citet{DanielyFrSi16} establish between a
compositional kernel and neural networks described by its skeleton. In
particular, we first explore the difference in classification accuracy between
the kernel and the network. Then, since training under the kernel is simply a
convex problem, we ask whether its relative accuracy across architectures
predicts well the relative performance of analogous fully-trained networks.

\begin{table}[t!]%[scale=0.8]
\begin{center}
	\begin{tabular}{
		|c
		|p{0.1\textwidth}
		|p{0.1\textwidth}
		|p{0.1\textwidth}
		|p{0.1\textwidth}|}
		\hline
		Arch & RFS & Net & Gap & Rank
		\\ \hline
		{\color{blue}4},{\color{red}5} & 72.88 & 77.09 & 4.21 & 1
		\\ \hline
		{\color{blue}6},{\color{red}4} & 72.42 & 77.06 & 4.64 & 2
		\\ \hline
		{\color{blue}6},{\color{red}6} & 72.25 & 75.52 & 3.27 & 4 (-1)
		\\ \hline
		{\color{blue}4},{\color{red}4} & 71.49 & 74.94 & 3.45 & 6 (-2)
		\\ \hline
		{\color{blue}5},{\color{red}4} & 71.45 & 76.78 & 5.33 & 3 (+2)
		\\ \hline
		{\color{blue}4},{\color{red}6} & 70.62 & 74.24 & 3.62 & 7 (-1)
		\\ \hline
		{\color{blue}5},{\color{red}6} & 70.39 & 75.38 & 4.99 & 5 (+2)
		\\ \hline
		{\color{blue}6},{\color{red}5} & 70.14 & 74.23 & 4.09 & 8
		\\ \hline
		{\color{blue}5},{\color{red}5} & 69.64 & 73.57 & 3.93 & 9
		\\ \hline
		Full & 60.01 & 55.09 & -4.92 & 10
		\\ \hline
		\end{tabular}
		\caption{Comparison of test accuracy of kernel learning with RFS and
		neural networks. For each architecture test we provide its ranking in
		terms of performance and relative rank ordering.
		The column designated as ``Arch'' describes the convolution
		size as pairs {\color{blue}$a$},{\color{red}$b$} where the first layer
		convolution is of size {\color{blue}$a$}x{\color{blue}$a$} and the second
		is {\color{red}$b$}x{\color{red}$b$}.}
		\label{tab:arch-and-nn}
		\end{center}
\end{table}

For the experiment, we considered several structures and trained both a
corresponding networks and the compositional kernel through an RFS
approximation. We again use the CIFAR-10 dataset, with a standard data
augmentation pipeline \citep{krizhevsky2012imagenet}: random 24x24 pixel crop,
random horizontal flip, random brightness, saturation, and contrast delta,
per-image whitening, and per-patch PCA.  In the test set, no data augmentation
is applied, and images are center-cropped to 24x24.  We generated $10^6$
random features, and trained for 120 epochs with
AdaGrad~\citep{duchi2011adaptive} and a manually tuned initial learning rate
among $\{25,50,100,200\}$.

The convolutional architectures are of the same kind as the deep variety in
Sec.~\ref{sec:concentration}, i.e.\ two convolutions with a size between
4x4 and 6x6 and a stride of 2, followed by a fully-connected layer.  The
fully-connected architecture is the shallow structure described in
Section~\ref{sec:concentration}.  All activations are ReLU.  The per-patch PCA
preprocessing step projects patches to the top $q$ principal components where
$q$ is 10, 12, and 16, respectively, for first-layer convolutions of size 4,
5, and 6, respectively.  The neural networks are generated from the kernel's
skeleton by node replication (i.e.\ producing channels) at a rate of 64
neurons per skeleton node for convolutions and 384 for fully-connected layers.
We use the typical random Gaussian initialization
\citep{glorot2010understanding}, 200 epochs of AdaGrad with a manually tuned
initial learning rate among $\{.0002,.0005,.001,.002,.005\}$.

%\begin{wrapfigure}{r}{0.5\textwidth}
\begin{figure}[t]
  \begin{center}
    %\centerline{\includegraphics[scale=.55]{correlation}}
  \includegraphics[scale=.55]{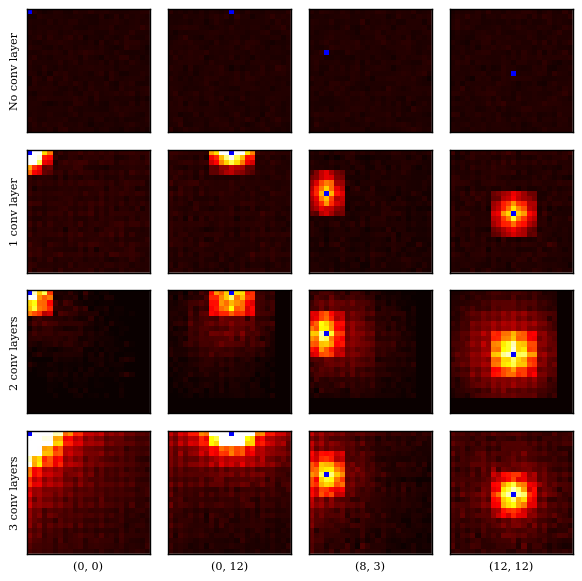}
  \caption{\small Illustration of the local structure captured by the our
    random feature scheme.  We generated 500,000 random features
    corresponding to four kernels - a flat kernel with one fully
    connected hidden layer, and three deep kernels with one, two and
    three convolution layers with ReLU activations.  The figures above
    show the correlations of 4 pixels (in blue) with all other
    pixels; lighter colors denote higher correlation.}
  \label{fig:correlation}
  \end{center}
\end{figure}
%\end{wrapfigure}

Test set accuracies are given in Table~\ref{tab:arch-and-nn}, annotated with
how well approximate RFS learning accuracies rank those of the trained
networks. We would like to underscore that, for convolutional kernels,
networks consistently outperformed the kernel.  Meanwhile, the fully-connected
kernel actually outperforms the corresponding network.  RFS learning moreover
ranks the top two networks correctly (as well as the bottom three).  This
observation is qualitatively in line with earlier findings of
\citet{saxe2011random}, who final-layer training of a randomly initialized
network to rank fully-trained networks. Indeed, from~\citet{DanielyFrSi16}, we
know that the random network approach is an alternative random feature map for
the compositional kernel.

%This seems not surprising as we all know that networks are the ultimate
%learning method.

\paragraph{Visualization of hierarchical random features.}

In Figure \ref{fig:correlation}, we illustrate how our random feature scheme
is able to accommodate local structures that are fundamental to image data.
We chose 4 different networks of varying depths, and generated 500,000 random
features for each.  We then computed for each pixel, the probability that it
co-occurs in a random feature with any other pixel, by measuring the
correlation between their occurrences.  As expected, for the flat kernel, the
correlation between any pair of pixels was the same.  However for deeper
kernels, nearby pixels co-occur more often in the random features, and the
degree of this co-occurrence is shown in the figure. As we increase the depth,
different tiers of locality appear. The most intense correlations share a
common first-layer convolution, while moderate correlation share only a
second-layer convolution. Lastly, mild correlation share no convolutions.

%\newpage

\bibliography{bib}

\begin{thebibliography}{22}
\providecommand{\natexlab}[1]{#1}
\providecommand{\url}[1]{\texttt{#1}}
\expandafter\ifx\csname urlstyle\endcsname\relax
  \providecommand{\doi}[1]{doi: #1}\else
  \providecommand{\doi}{doi: \begingroup \urlstyle{rm}\Url}\fi

\bibitem[Bach(2014)]{bach2014breaking}
F.~Bach.
\newblock Breaking the curse of dimensionality with convex neural networks.
\newblock \emph{arXiv:1412.8690}, 2014.

\bibitem[Bach(2015)]{bach2015equivalence}
F.~R. Bach.
\newblock On the equivalence between quadrature rules and random features.
\newblock \emph{CoRR}, abs/1502.06800, 2015.
\newblock URL \url{http://arxiv.org/abs/1502.06800}.

\bibitem[Bo et~al.(2011)Bo, Lai, Ren, and Fox]{bo2011object}
L.~Bo, K.~Lai, X.~Ren, and D.~Fox.
\newblock Object recognition with hierarchical kernel descriptors.
\newblock In \emph{Computer Vision and Pattern Recognition (CVPR), 2011 IEEE
  Conference on}, pages 1729--1736. IEEE, 2011.

\bibitem[Cho and Saul(2009)]{cho2009kernel}
Y.~Cho and L.K. Saul.
\newblock Kernel methods for deep learning.
\newblock In \emph{Advances in neural information processing systems}, pages
  342--350, 2009.

\bibitem[Daniely(2017)]{daniely2017}
Amit Daniely.
\newblock Sgd learns the conjugate kernel class of the network.
\newblock \emph{arXiv:1702.08503}, 2017.

\bibitem[Daniely et~al.(2016)Daniely, Frostig, and Singer]{DanielyFrSi16}
Amit Daniely, Roy Frostig, and Yoram Singer.
\newblock Toward deeper understanding of neural networks: The power of
  initialization and a dual view on expressivity.
\newblock In \emph{Advances In Neural Information Processing Systems}, pages
  2253--2261, 2016.

\bibitem[Duchi et~al.(2011)Duchi, Hazan, and Singer]{duchi2011adaptive}
John Duchi, Elad Hazan, and Yoram Singer.
\newblock Adaptive subgradient methods for online learning and stochastic
  optimization.
\newblock \emph{Journal of Machine Learning Research}, 12\penalty0
  (Jul):\penalty0 2121--2159, 2011.

\bibitem[Glorot and Bengio(2010)]{glorot2010understanding}
X.~Glorot and Y.~Bengio.
\newblock Understanding the difficulty of training deep feedforward neural
  networks.
\newblock In \emph{International conference on artificial intelligence and
  statistics}, pages 249--256, 2010.

\bibitem[Grauman and Darrell(2005)]{grauman2005pyramid}
K.~Grauman and T.~Darrell.
\newblock The pyramid match kernel: Discriminative classification with sets of
  image features.
\newblock In \emph{Tenth IEEE International Conference on Computer Vision},
  volume~2, pages 1458--1465, 2005.

\bibitem[Kar and Karnick(2012)]{kar2012random}
P.~Kar and H.~Karnick.
\newblock Random feature maps for dot product kernels.
\newblock \emph{arXiv:1201.6530}, 2012.

\bibitem[Krizhevsky et~al.(2009)Krizhevsky, Sutskever, and
  Hinton]{KrizhevskySuHi09}
A.~Krizhevsky, I.~Sutskever, and G.~Hinton.
\newblock Learning multiple layers of features from tiny images.
\newblock Technical report, University of Toronto, 2009.

\bibitem[Krizhevsky et~al.(2012)Krizhevsky, Sutskever, and
  Hinton]{krizhevsky2012imagenet}
A.~Krizhevsky, I.~Sutskever, and G.E. Hinton.
\newblock Imagenet classification with deep convolutional neural networks.
\newblock In \emph{Advances in neural information processing systems}, pages
  1097--1105, 2012.

\bibitem[Mairal(2016)]{Mairal16}
J.~Mairal.
\newblock {End-to-End Kernel Learning with Supervised Convolutional Kernel
  Networks}.
\newblock In \emph{{Advances in Neural Information Processing Systems (NIPS)}},
  2016.

\bibitem[Mairal et~al.(2014)Mairal, Koniusz, Harchaoui, and
  Schmid]{mairal2014convolutional}
J.~Mairal, P.~Koniusz, Z.~Harchaoui, and Cordelia Schmid.
\newblock Convolutional kernel networks.
\newblock In \emph{Advances in Neural Information Processing Systems}, pages
  2627--2635, 2014.

\bibitem[Pennington et~al.(2015)Pennington, Yu, and
  Kumar]{pennington2015spherical}
J.~Pennington, F.~Yu, and S.~Kumar.
\newblock Spherical random features for polynomial kernels.
\newblock In \emph{Advances in Neural Information Processing Systems}, pages
  1837--1845, 2015.

\bibitem[Rahimi and Recht(2007)]{RahimiRe07}
A.~Rahimi and B.~Recht.
\newblock Random features for large-scale kernel machines.
\newblock In \emph{NIPS}, pages 1177--1184, 2007.

\bibitem[Saxe et~al.(2011)Saxe, Koh, Chen, Bhand, Suresh, and
  Ng]{saxe2011random}
A.~Saxe, P.W. Koh, Z.~Chen, M.~Bhand, B.~Suresh, and A.Y. Ng.
\newblock On random weights and unsupervised feature learning.
\newblock In \emph{Proceedings of the 28th International Conference on Machine
  Learning (ICML-11)}, pages 1089--1096, 2011.

\bibitem[Sch{\"o}lkopf et~al.(1998{\natexlab{a}})Sch{\"o}lkopf, Burges, and
  Smola]{ScholkopfBuSm98}
B.~Sch{\"o}lkopf, C.~Burges, and A.~Smola, editors.
\newblock \emph{Advances in Kernel Methods - Support Vector Learning}.
\newblock MIT Press, 1998{\natexlab{a}}.

\bibitem[Sch{\"o}lkopf et~al.(1998{\natexlab{b}})Sch{\"o}lkopf, Simard, Smola,
  and Vapnik]{scholkopf1998prior}
B.~Sch{\"o}lkopf, P.~Simard, A.~Smola, and V.~Vapnik.
\newblock Prior knowledge in support vector kernels.
\newblock In \emph{Advances in Neural Information Processing Systems 10}, pages
  640--646. MIT Press, 1998{\natexlab{b}}.

\bibitem[Shalev-Shwartz and Ben-David(2014)]{shalev2014understanding}
S.~Shalev-Shwartz and S.~Ben-David.
\newblock \emph{Understanding Machine Learning: From Theory to Algorithms}.
\newblock Cambridge University Press, 2014.

\bibitem[Vapnik(1998)]{Vapnik98}
V.~N. Vapnik.
\newblock \emph{Statistical Learning Theory}.
\newblock Wiley, 1998.

\bibitem[Vapnik(1995)]{Vapnik95}
V.N. Vapnik.
\newblock \emph{The Nature of Statistical Learning Theory}.
\newblock Springer, 1995.

\end{thebibliography}
\end{document}